\documentclass{article}

\usepackage{microtype}
\usepackage{graphicx}
\usepackage{subcaption}
\usepackage{booktabs} 

\usepackage{hyperref}

\usepackage[preprint]{icml}

\usepackage{amsmath}
\usepackage{amssymb}
\usepackage{mathtools}
\usepackage{amsthm}

\DeclareMathOperator{\KL}{KL}
\DeclareMathOperator{\Bern}{Bern}

\usepackage[capitalize,noabbrev]{cleveref}

\theoremstyle{plain}
\newtheorem{theorem}{Theorem}[section]
\newtheorem{proposition}[theorem]{Proposition}
\newtheorem{lemma}[theorem]{Lemma}
\newtheorem{corollary}[theorem]{Corollary}
\theoremstyle{definition}

\theoremstyle{remark}

\usepackage[textsize=tiny]{todonotes}

\icmltitlerunning{Feasibility Limits of Human Preference Evaluation}

\begin{document}

\twocolumn[
  \icmltitle{How Many Human Judgments Are Enough? \\ Feasibility Limits of Human Preference Evaluation}

  \icmlsetsymbol{equal}{*}

  \begin{icmlauthorlist}
    \icmlauthor{Wilson Y. Lee}{equal,yyy}
  \end{icmlauthorlist}

  \icmlaffiliation{yyy}{Independent Researcher}

  \icmlcorrespondingauthor{Wilson Y. Lee}{wilson.yenhsun.lee@gmail.com}

  \icmlkeywords{Human evaluation, Preference evaluation, Statistical power analysis, LLM evaluation, Sample size estimation, Benchmark design, Chatbot Arena}

  \vskip 0.3in
]



\printAffiliationsAndNotice{}  

\begin{abstract}
Human preference evaluations are widely used to compare generative models, yet it remains unclear how many judgments are required to reliably detect small improvements. We show that when preference signal is diffuse across prompts (i.e., all prompt types are similarly informative), proportional allocation is minimax-optimal: no allocation strategy substantially improves detectability. Empirical analysis of large-scale human preference datasets shows that 
most comparisons fall into this diffuse regime, exhibiting small preference margins that require far more judgments than typically collected, even in well-sampled comparisons. These limits persist across evaluation protocols and modalities, including chat, image generation, and code generation with execution feedback. In contrast, curated benchmarks that reduce prompt-induced variability systematically induce larger margins and improve detectability through a $1.5\times$ reduction in prompt-level variance. Our results show that inconclusive or negative human evaluation outcomes frequently reflect underpowered evaluation rather than model equivalence, underscoring the need to account explicitly for effect size, budget, and protocol design.
\end{abstract}

\section{Introduction}

Human preference studies are widely used to evaluate generative models, particularly for open-ended tasks where automated metrics poorly align with perceived quality \citep{gehrmann2021gem, chiang2024chatbot, zheng2023judging, jiang2025artificial}. Pairwise human judgments directly compare models and are central to modern evaluation pipelines, including public leaderboards and curated benchmarks. However, such evaluations are expensive, slow, and often inconclusive, especially as model improvements become incremental. As a result, practitioners routinely face a practical question: “how many human judgments are required to reliably detect a meaningful improvement?”

Prior work has noted that human evaluation results are often underpowered or difficult to interpret when effect sizes are small \citep{gehrmann2021gem, dror2018hitchhiker, card2020little, elangovan2024considers, miller2024addingerrorbarsevals}, even under careful uncertainty estimation. Classical power analysis assumes a single homogeneous effect size and i.i.d.\ samples \citep{lehmann2005testing, casella2002statistical}, which are violated by human preference evaluation where effect sizes vary across prompts \citep{zheng2023judging}. This motivates a feasibility analysis linking evaluation budget, effect size, and detectability under heterogeneity.

\begin{figure}[t]
  \centering
  \includegraphics[width=0.95\linewidth]
  {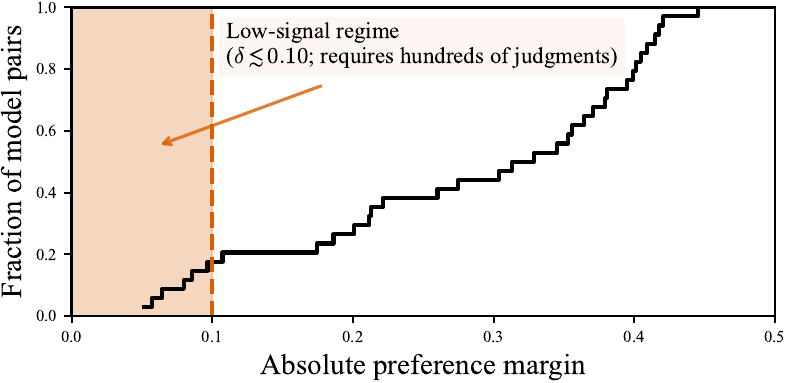}
  \caption{Distribution of empirical absolute preference margins $|\delta|$ for well-sampled Chatbot Arena model pairs ($\geq$200 decisive judgments). A substantial fraction lie in a low-signal regime ($|\delta| \lesssim 0.10$), where reliable detection requires hundreds of judgments.}
  \label{fig:arena_margins}
\end{figure}

In this work, we study human preference evaluation through the lens of statistical hypothesis testing under heterogeneous prompts and finite budgets. We show that reliable detection is governed by the total KL divergence accumulated across judgments, and that evaluation protocols  determine how this budget is distributed across prompt types. This perspective explains why open-ended evaluations such as Chatbot Arena~\citep{chiang2024chatbot} often struggle to detect small improvements despite large sample sizes, and clarifies when protocol design can (and cannot) alter feasibility limits.

\noindent\textbf{Our contributions.} We characterize feasibility in human preference evaluation under prompt-induced heterogeneity, addressing three questions: (1) In what regimes is reliable detection intrinsically difficult? (2) Can adaptive prompt allocation improve detectability, holding the protocol fixed? (3) How does protocol design reshape feasibility limits?

Our main theoretical result answers (2) negatively. When preference signal is diffuse across prompts (e.g., open-ended evaluation like Chatbot Arena), proportional allocation is minimax-optimal. In other words, no allocation strategy can rescue underpowered comparisons when the fundamental issue is insufficient total signal. However, when signal is concentrated in a small subset of prompt types, two-stage screening can provide measurable gains.

Empirically, we ground this framework using Chatbot Arena and MT-Bench~\citep{zheng2023judging}. Many well-sampled Arena comparisons lie in low-signal regimes requiring far more judgments than typically collected: 17.6\% of model pairs with $\geq$200 judgments have preference margins $|\delta| \leq 0.10$, implying $\geq$500 judgments for 90\% detection power (Figure~\ref{fig:arena_margins}). MT-Bench achieves higher detectability through a $1.5\times$ reduction in prompt-level variance---providing a mechanistic account of why protocol design, not annotation volume, determines which comparisons are resolvable.

\section{Problem Setup and Definitions}
\label{sec:setup}

We consider the problem of comparing two generative models, denoted $A$ and $B$, using pairwise human preferences. Given a prompt sampled from a fixed distribution and responses generated by both models, a human evaluator is asked to indicate which response they prefer, or whether the responses are tied.

Throughout this work, we focus on \emph{decisive} comparisons and discard ties. Each human judgment is therefore represented as a binary random variable
\begin{equation}
Y = \mathbf{1}\!\left[\text{human prefers model } B \text{ over model } A\right].
\end{equation}

where $Y=1$ indicates a preference for $B$ and $Y=0$ indicates a
preference for $A$. We define
\begin{equation}
p = \Pr(Y=1),
\end{equation}
the probability that a randomly sampled evaluator prefers $B$ over $A$ under the given evaluation protocol.

\textbf{Ties and best-case analysis.}
Human evaluators may also report ties, indicating no clear preference between the two responses. Throughout this work, we discard ties and condition on decisive outcomes in order to obtain a binary preference signal. This choice corresponds to a \emph{best-case} analysis. Conditioning on decisive judgments maximizes the per-judgment directional information available for distinguishing $p \neq 1/2$. Incorporating ties via a multinomial or ordinal preference model can only reduce the effective information per judgment; therefore, it cannot improve detectability. We empirically verify robustness to alternative tie encodings in Appendix~\ref{app:ties}.

\subsection{Preference Margin as Effect Size}

The central quantity governing detectability in pairwise human evaluation is the magnitude of the preference margin. We define the preference margin as
\begin{equation}
\delta = p - \tfrac{1}{2},
\end{equation}
which measures the deviation from indifference. Larger values of $|\delta|$ correspond to differences that are easier for human evaluators to perceive, while values of $|\delta|$ close to zero correspond to subtle qualitative differences.

From a statistical perspective, detecting an improvement reduces to distinguishing the null hypothesis $p = 0.5$ from the alternative $p \neq 0.5$ given finite noisy observations. The fundamental quantity governing detectability is the KL divergence between these hypotheses. For Bernoulli preferences near $p = 0.5$, the per-judgment KL divergence scales as $(p - \tfrac12)^2$ up to constants, so empirical preference margins serve as observable proxies for the KL budget accumulated by an evaluation protocol.

\subsection{Near-Tie (Low-Signal) Comparisons}

We refer to comparisons with small absolute preference margins as \emph{near-tie} or \emph{low-signal} comparisons. Formally, for a chosen threshold $\tau \in (0, 0.5)$, we define the near-tie subset as
\begin{equation}
|\hat{p} - \tfrac{1}{2}| < \tau,
\label{eq:neartie}
\end{equation}
where $\hat{p}$ is the empirical estimate of $p$ from observed judgments. Rather than treating $\tau$ as a universal constant, we treat it as a sensitivity parameter and report results across multiple thresholds where appropriate.

We frame human evaluation as a hypothesis-testing problem,
\begin{equation}
H_0: p = 0.5 \quad \text{vs.} \quad H_1: p \neq 0.5,
\end{equation}
and study detectability as a function of evaluation budget. This framing clarifies that inconclusive human evaluations are not evidence of model equivalence—they may simply reflect insufficient budget relative to the underlying effect size. With $n$ independent judgments, standard binomial tests imply that rejection probability depends on $n$, the significance level $\alpha$, and the preference margin $|\delta|$.

Rather than fixing $n$, we characterize \emph{detectability curves} mapping budget to detection probability. These curves make explicit the tradeoff between effect size and sample size and clarify the interpretation of inconclusive results: when $|\delta|$ is small, detectability saturates and additional budget yields diminishing returns.

\section{Theory: Feasibility Limits and Their Consequences}
\label{sec:kl_theory}
Standard power analysis assumes homogeneous effect size $\delta$; the $\delta^{-2}$ sample complexity scaling is classical. Our contribution is characterizing when prompt-induced heterogeneity permits adaptive gains and when it does not. We show that proportional allocation is minimax-optimal in diffuse regimes (Theorem~\ref{thm:minimax-allocation}), while two-stage screening helps only under sufficient signal concentration (Theorem~\ref{thm:stochastic-adaptive}). Section~\ref{sec:implementation-guide} translates these results into practical guidance.

\subsection{Setup}
\label{sec:theory_setup}
We study hypothesis testing under prompt-induced heterogeneity, where individual judgments carry unequal statistical information. At round $t$,
\begin{equation}
Y_t \sim \mathrm{Bernoulli}\!\left(\tfrac12+\delta_t\right), \qquad \delta_t\in[-\tfrac14,\tfrac14],
\label{eq:kl_model}
\end{equation}
where conditional on $\delta_{1:B}$ the observations are independent.
Let $P_{\delta}$ denote the joint law under $\delta_{1:B}$ and $P_0$ the null law
$\delta_t\equiv 0$.

We define the \emph{KL budget} of an evaluation stream as
\begin{equation}
\begin{aligned}
\mathsf{K}(\delta_{1:B})
&:= \mathrm{KL}(P_{\delta}\,\|\,P_0) \\
&= \sum_{t=1}^B
\mathrm{KL}\!\left(
\mathrm{Bern}\!\left(\tfrac12+\delta_t\right)
\,\middle\|\,
\mathrm{Bern}\!\left(\tfrac12\right)
\right).
\end{aligned}
\label{eq:kl_budget}
\end{equation}

All feasibility bounds in Section~\ref{sec:kl_theory} should be interpreted as best-case limits: any dependence between judgments can only reduce effective sample size and make  detection harder.

\begin{lemma}[Quadratic KL scaling]
\label{lem:kl_quadratic}
If $|\delta|\le\tfrac14$, then
\[
2\,\delta^2
\;\le\;
\KL\!\left(\Bern(\tfrac12+\delta)\,\middle\|\,\Bern(\tfrac12)\right)
\;\le\;
\tfrac94\,\delta^2,
\]
and consequently
\[
2\sum_{t=1}^B\delta_t^2
\;\le\;
\mathsf{K}(\delta_{1:B})
\;\le\;
\tfrac94\sum_{t=1}^B\delta_t^2 .
\]
\end{lemma}

\subsection{Minimax feasibility under adversarial heterogeneity}
We test $H_0:P_0$ versus the composite alternative
\[
H_1:\ \mathsf{K}(\delta_{1:B}) \ge \mathsf{K},
\]
for a target $\mathsf{K}>0$. A test $\phi$ has type-I error
$\alpha(\phi)=\Pr_0(\phi=1)$ and worst-case type-II error
$\beta(\phi)=\sup_{\mathsf{K}(\delta)\ge \mathsf{K}}\Pr_{\delta}(\phi=0)$.

\begin{theorem}[Minimax lower bound]
\label{thm:kl_lower_main}
For any test $\phi$ and any $\mathsf{K}>0$,
\[
\sup_{\mathsf{K}(\delta)\ge \mathsf{K}}
\big(\Pr_0[\phi=1]+\Pr_{\delta}[\phi=0]\big)
\;\ge\; \tfrac12 e^{-\mathsf{K}}.
\]
Thus, achieving total error $\le \alpha+\beta$ requires
$\mathsf{K}\gtrsim \log\!\frac{1}{\alpha+\beta}$.
\end{theorem}

\begin{theorem}[Matching upper bound]
\label{thm:kl_upper_main}
There exist universal constants $c,C>0$ such that for any $\mathsf{K}>0$
there exists a test $\phi_{\mathsf{K}}$ satisfying
\[
\sup_{\mathsf{K}(\delta)\ge C\mathsf{K}}
\big(\Pr_0[\phi_{\mathsf{K}}=1]+\Pr_{\delta}[\phi_{\mathsf{K}}=0]\big)
\;\le\; e^{-c\mathsf{K}}.
\]
\end{theorem}

\textbf{Proof sketch.}
The lower bound follows from the Bretagnolle--Huber inequality. The upper bound uses a likelihood-ratio test thresholded at $\mathsf{K}/2$, with bounded increments yielding concentration when $\mathrm{KL}(P_{\delta}\|P_0)\ge C\mathsf{K}$. Full proofs appear in Appendix~\ref{app:kl_proofs}. Together, these bounds show that the KL budget is both necessary and sufficient for detection, up to constants—establishing KL divergence as the right currency for reasoning about evaluation feasibility.

\subsection{Stochastic prompt heterogeneity}
\label{sec:kl_random_effects}
We now consider a random-effects model in which
$\delta_1,\dots,\delta_B \stackrel{i.i.d.}{\sim}\mathcal{D}$ for a distribution
$\mathcal{D}$ supported on $[-\tfrac14,\tfrac14]$. Unless stated otherwise, all expectations in this subsection are taken with respect to the random draw $\delta_{1:B} \sim \mathcal{D}^B$.

\begin{proposition}
\label{prop:random_effects}
Let $\mu_2=\mathbb{E}_{\delta\sim\mathcal{D}}[\delta^2]$.
There exist universal constants $c,C>0$ such that:
\begin{enumerate}
\item $\mathbb{E}[\mathsf{K}(\delta_{1:B})]\asymp B\,\mu_2$.
\item The test $\phi_{\mathsf{K}}$ from Theorem~\ref{thm:kl_upper_main} achieves
\[
\mathbb{E}\!\left[\Pr_0(\phi_{\mathsf{K}}=1)+\Pr_{\delta}(\phi_{\mathsf{K}}=0)\right]
\le e^{-c\mathsf{K}}
\]
whenever $B\mu_2\ge C\mathsf{K}$.
\end{enumerate}
\end{proposition}

\textbf{Implication.} Up to constants, feasibility is governed by the accumulated KL budget, equivalently by $B\,\mathbb{E}[\delta^2]$ (Lemma~3.1, Proposition~3.4).

\subsection{Optimal Allocation of Evaluation Budget}
\label{sec:protocol-optimality}

Evaluation protocols allocate a finite judgment budget across heterogeneous prompts; we ask whether this allocation can improve detectability beyond KL-budget limits.

\subsubsection{Adversarial heterogeneity}

Let $m$ prompt types have incidence weights $w\in\Delta^m$. A protocol $\pi$ allocates a total budget $B$ judgments, producing counts $N_1,\dots,N_m$ with $\sum_i N_i=B$. Define the \emph{allocation bottleneck}
\[
\Lambda(\pi) := \mathbb E_\pi\!\left[\min_{i\in[m]} \frac{N_i}{w_i}\right].
\]
For a test $\varphi$, define the worst-case risk
\[
\mathcal R(\pi,\varphi)
:= \Pr_0^\pi(\varphi=1)
+ \sup_{\delta\in\mathcal H_1(\mu_2)} \Pr_\delta^\pi(\varphi=0),
\]
where
\(
\mathcal H_1(\mu_2)
=\{\delta:\sum_i w_i\delta_i^2\ge\mu_2\}.
\)

\begin{theorem}[Minimax optimality of proportional allocation]
\label{thm:minimax-allocation}
There exist universal constants $c,C>0$ such that for any allocation policy $\pi$ and test $\varphi$,
\[
\mathcal R(\pi,\varphi)
\;\ge\;
\tfrac12 \exp\!\big(-C\,\mu_2\,\Lambda(\pi)\big).
\]
Consequently,
\[
\inf_{\pi,\varphi} \mathcal R(\pi,\varphi)
\;\ge\;
\tfrac12 \exp(-C B\mu_2).
\]
Moreover, proportional allocation with $N_i\approx Bw_i$ satisfies
$\Lambda(\pi)\ge cB$ and achieves this rate up to constants.
\end{theorem}

\textbf{Interpretation.} This result implies that in open-ended evaluation settings, where signal is spread diffusely across many prompt types, sophisticated allocation strategies cannot rescue underpowered comparisons. The fundamental bottleneck is the total KL budget, not how it is distributed.

\subsubsection{Stochastic heterogeneity}

Assume $\delta_1,\dots,\delta_m\stackrel{\text{i.i.d.}}{\sim}D$ with
$\mu_2=\mathbb E[\delta^2]$.
Let $S_q(\delta)$ denote the indices of the top $\lceil qm\rceil$
values of $\delta_i^2$, and define the concentrated signal
\[
\mu_{2,q}(\delta)
:= \sum_{i\in S_q(\delta)}
\frac{w_i}{\sum_{j\in S_q(\delta)} w_j}\,\delta_i^2.
\]

While Arena is not adversarial, Theorem 3.5 provides a worst-case bound showing that sufficiently diffuse signal alone is enough to preclude gains from adaptive allocation.

\begin{theorem}[Adaptive gains under stochastic heterogeneity]
\label{thm:stochastic-adaptive}
There exists a two-stage adaptive allocation policy
$\pi_{\mathrm{2stage}}$ and constants $c,c'>0$ such that, for all
sufficiently large $B$,
\[
\mathbb E_\delta\!\left[\mathcal R(\pi_{\mathrm{2stage}},\varphi)\right]
\;\le\;
\exp\!\big(-c\,B\,\mathbb E[\mu_{2,q}(\delta)]\big),
\]
while for proportional allocation $\pi_{\mathrm{prop}}$,
\[
\inf_\varphi
\mathbb E_\delta\!\left[\mathcal R(\pi_{\mathrm{prop}},\varphi)\right]
\;\ge\;
\exp(-c' B\mu_2).
\]
\end{theorem}

\textbf{Implication.}
See Appendix~\ref{app:implications} for interpretation and practical consequences.

\begin{corollary}[Practical parameter choices for two-stage allocation]
\label{cor:two-stage-params}
Suppose there are $m$ prompt types and total budget $B$. Choose a screening budget $b = c \log m$ with $c \ge 10$ and a retention fraction $q \in [0.1, 0.3]$. Then screening succeeds with probability at least $1 - 2\exp(-c_1 b)$ for a universal constant $c_1>0$, and conditional on correct screening, the two-stage allocation achieves an error exponent proportional to $B \mu_{2,q}$, assuming repeated judgments within prompt types.
\end{corollary}

\textbf{Two-stage screening.}
The policy $\pi_{\mathrm{2stage}}$ allocates a small uniform screening budget to estimate $\delta_i^2$, then concentrates the remaining budget on the highest signal types. Algorithmic details and proofs are given in Appendix~\ref{app:proofs_34}.

\subsection{A Practitioner’s Guide to Human Evaluation Under Budget Constraints}
\label{sec:implementation-guide}

We translate the results of Sections~\ref{sec:theory_setup}--\ref{sec:protocol-optimality} into a concrete decision framework for designing human preference evaluations under finite budgets.

\textbf{Decision procedure.}
Given an evaluation goal and a fixed judgment budget:
\begin{enumerate}
    \item \textbf{Clarify the evaluation goal.}
    If the goal is deployment realism or broad prompt coverage, use open-ended evaluations with diverse prompts. If the goal is relative ranking or ablation under limited budgets, prefer curated or task-focused benchmarks.
    
    \item \textbf{Pilot estimation.}
    Collect a pilot of $n_0 \in [30, 100]$ decisive judgments and estimate $\hat{\delta} = \hat{p} - \tfrac{1}{2}$. If $|\hat{\delta}| \lesssim 0.05$, detection typically requires ${\sim}10^3$ judgments; consider a curated protocol or report the result as underpowered.
    
    \item \textbf{Feasibility check.}
    Use the pilot estimate $\hat{\delta}$ to assess feasibility via Eq.~\ref{eq:closed_form_simplified}: $n \approx 2.63/\hat{\delta}^2$ (for commonly used $\alpha = 0.05$, 90\% power). The theory in Section~\ref{sec:kl_theory} justifies the $\delta^{-2}$ scaling but not the precise constant.

    \item \textbf{Protocol choice.}
    If the required $n$ exceeds the available budget, either switch to a curated protocol that amplifies margins or explicitly report the comparison as underpowered rather than inconclusive.

    \item \textbf{Final consistency check.}
    Verify that a pilot estimate $\hat{\delta}$ has been obtained and that the implied budget $n \approx 2.63/\hat{\delta}^2$ is feasible. Otherwise, report the comparison as underpowered.

\end{enumerate}

\textbf{Worked example (open-ended).}
Consider a Chatbot Arena pair with pilot win rate $\hat{p}=0.56$ from $n_0=50$, yielding $\hat{\delta}=0.06$ and an implied budget requirement of $n \approx 2.63/0.06^2 \approx 731$. Under typical Arena budgets ($B \le 500$), non-detection should therefore be interpreted as a budget limitation rather than evidence of equivalence.

\textbf{When to use two-stage allocation.}
Algorithm~\ref{alg:two-stage} helps only when signal concentrates in a few prompt types. Both protocols are mostly diffuse (Appendix~\ref{app:offline-replay}), though MT-Bench occasionally shows concentration:
\begin{center}
\footnotesize
\begin{tabular}{p{0.46\columnwidth}cc}
\toprule
Scenario & Proportional & Two-stage \\
\midrule
Diffuse signal (typical) & \checkmark & $\times$ \\
Concentrated signal & \checkmark & $\checkmark$ if $\kappa > 1.5$ \\
Small budget ($B < mb$) & \checkmark & $\times$ \\
Large budget, uneven signal & \checkmark & \checkmark \\
\bottomrule
\end{tabular}
\end{center}

In diffuse regimes, proportional allocation is minimax-optimal (Theorem~\ref{thm:minimax-allocation}); two-stage allocation yields gains only when signal concentration is high, as characterized by Theorem~\ref{thm:stochastic-adaptive} and verified empirically in Appendix~\ref{app:implications}.

\section{Method: Detectability Curves and Budget Estimation}
\label{sec:method}

We empirically instantiate the detectability framework of Section~\ref{sec:kl_theory} using subsampling-based power estimates. We summarize detectability outcomes across representative model pairs and budgets in the results section. These empirical patterns motivate the protocol-level implications discussed in Section~\ref{sec:budgeting}.

Rather than reporting a single test outcome at a fixed budget, we characterize \emph{detectability as a function of budget}. For a given dataset and model pair, we repeatedly subsample $n$ judgments with replacement and estimate the probability that a hypothesis test detects an improvement, yielding a \emph{detectability curve} that maps evaluation budget $n$ to detection probability. We test a Bernoulli preference rate $p$ against $p=0.5$; classical power analysis implies the required budget scales as $n=\Theta(\delta^{-2})$ for $\delta = |p-0.5|$. Guided by Section~\ref{sec:kl_theory}, we interpret $\delta$ as an observable proxy for per-judgment information, so detectability curves summarize feasibility as a function of both effect size and realized evaluation allocation.

\subsection{Robustness and Statistical Considerations}
\label{sec:robustness}

We estimate detectability curves using repeated subsampling with replacement, which directly estimates the probability that a hypothesis test succeeds at a given budget. This directly addresses the operational question of practical interest and yields results qualitatively similar to bootstrap-based alternatives. We discard ties to obtain a best-case binary preference signal, as ties provide no directional information. Appendix~\ref{app:ties} shows that ties concentrate in low-margin regimes and that alternative tie encodings reduce estimated margins and increase required evaluation budgets, without altering qualitative conclusions.

\begin{figure}[ht]
  \centering
  \includegraphics[width=0.95\linewidth]{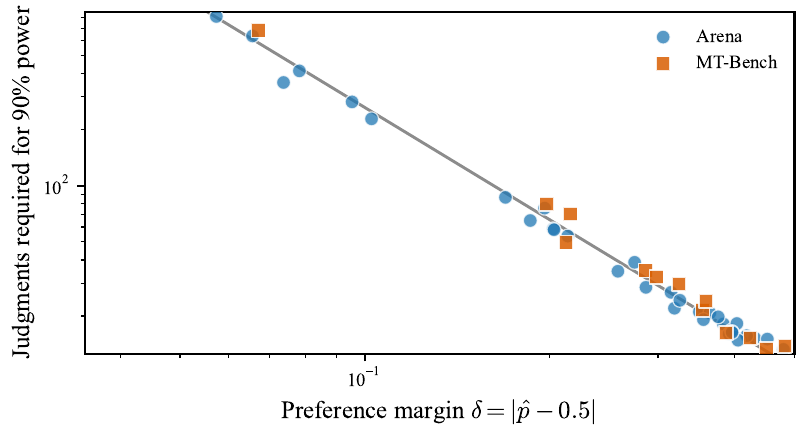}
  \caption{
  Empirical validation of theoretical scaling. Each point corresponds to a model comparison, plotted by its preference margin and the implied evaluation budget required for detection power. The solid curve shows the predicted asymptotic $\delta^{-2}$ scaling from Eq.~\ref{eq:closed_form_budget}.}
  \label{fig:theory_practice}
\end{figure}

\section{Empirical Evaluation}
\label{sec:emperical_eval}

We evaluate our framework across large-scale human preference datasets that differ in prompt distribution, annotation protocol, and evaluation goals. These datasets allow us to characterize how perceptual effect size and evaluation budget interact in open-ended and curated settings (Figure~\ref{fig:theory_practice}). To assess generality beyond chat-based text evaluation, we analyze the text-to-image preference dataset \textbf{Pick-a-Pic} \citep{Kirstain2023PickAPic} and the code generation platform \textbf{BigCodeArena} \citep{zhuo2025bigcodearena}, which provides execution feedback. Both exhibit comparable near-tie regimes: 14.2\% of image comparisons and 41\% of code comparisons fall into $|\hat{\delta}| \leq 0.10$, requiring hundreds to thousands of judgments for reliable detection (Appendix~\ref{app:pickapic}, \ref{app:bigcodearena}). These findings demonstrate that feasibility limits persist across modalities, even when partial ground truth is available.

\subsection{Selection Effects and Feasibility in Arena}
\label{sec:arena-independence}

Table~\ref{tab:tail-quantiles} summarizes the distribution of absolute preference margins $|\delta|$ among well-sampled model pairs and the implied number of decisive judgments required for reliable detection. Median margins among well-sampled model pairs under each protocol are relatively large under both Arena and MT-Bench, implying that many comparisons are easily detectable with modest budgets. This reflects strong selection effects: comparisons exhibiting clear signal tend to attract sustained evaluation effort and become well-sampled.

\begin{table}[ht]
\centering
\caption{Tail quantiles of absolute preference margins $|\delta|$ and the implied number of decisive human judgments required to achieve 90\% detection power, computed using Eq.~(8) with $\alpha=0.05$. Median margins are large due to selection effects, while feasibility limits arise in the lower tail.}
\label{tab:tail-quantiles}
\resizebox{\columnwidth}{!}{%
\begin{tabular}{llccc}
\toprule
Protocol & Statistic & $|\delta|$ & $n$ ($\alpha=0.05$) & $n$ ($\alpha=0.01$) \\
\midrule
Arena    & $p_{10}(|\delta|$) & 0.082 & 395  & 560  \\
Arena    & $p_{25}(|\delta|$) & 0.190 & 73   & 103  \\
Arena    & $p_{50}(|\delta|$) & 0.321 & 26   & 36   \\
\midrule
MT-Bench & $p_{10}(|\delta|$) & 0.047 & 1175 & 1664 \\
MT-Bench & $p_{25}(|\delta|$) & 0.187 & 75   & 106  \\
MT-Bench & $p_{50}(|\delta|$) & 0.285 & 32   & 46   \\
\bottomrule
\end{tabular}%
}
\end{table}

Feasibility limits, however, are governed by the lower tail of the margin distribution. Table~\ref{tab:near-tie} isolates this regime by conditioning on near-tie comparisons with $|\delta|\le\tau$. Even among well-sampled Arena pairs, a non-trivial fraction fall into the $|\delta|\le0.10$ regime, where the median conditional margin implies required budgets on the order of hundreds of judgments, exceeding typical Arena budgets (bootstrap in Appendix~\ref{app:robustness}). A similar pattern appears under MT-Bench: while margins are larger on average, near-tie comparisons still require thousands of judgments for reliable detection.

\begin{table}[ht]
\centering
\caption{Conditional feasibility for near-tie comparisons. Conditioning on low-signal comparisons with $|\delta|\le\tau$, we report the median margin within this subset and the implied budget for 90\% detection power. Proportion denotes the fraction of well-sampled model pairs under each protocol satisfying the near-tie condition.}
\label{tab:near-tie}
\resizebox{\columnwidth}{!}{%
\begin{tabular}{llcccc}
\toprule
Protocol & Condition & Count & Proportion & Median $|\delta|$ & $n$ ($\alpha=0.05$) \\
\midrule
Arena    & $|\delta|\le0.10$ & 6 & 0.176 & 0.072 & 506  \\
MT-Bench & $|\delta|\le0.10$ & 3 & 0.200 & 0.037 & 1878 \\
\bottomrule
\end{tabular}%
}
\end{table}

To assess whether violations of independence could explain deviations from the best-case feasibility bounds of Section~\ref{sec:kl_theory}, we compare classical i.i.d.\ standard errors to cluster-robust estimates allowing arbitrary correlation within prompts. The resulting variance inflation is negligible across all pairs (median ratio 1.000; 99th percentile below 1.001), indicating that detectability failures arise from small preference margins rather than violations of independence. Additional robustness checks are reported in Appendix~\ref{app:robustness}.

\subsection{Near-Tie Comparisons and Budget Requirements}
\label{sec:near-ties}

We analyze model comparisons with small absolute preference margins to characterize evaluation difficulty when qualitative differences are subtle. Conditioning on
\(
|\hat{p} - 0.5| < \tau
\)
isolates a low-signal regime in which statistical detectability is intrinsically limited.

Empirically, near-tie comparisons are not artifacts of selective evaluation. Estimated preference margins are stable as additional judgments accrue. Conditioning on an equal-exposure cohort (50 early judgments per pair) yields strongly correlated early and final margins (Spearman $\rho = 0.93$), with no evidence of drift toward lower separability (Appendix~\ref{app:robustness}). Near-ties are rare under equal exposure, but when they appear they persist after 200+ judgments, indicating that low-signal comparisons are intrinsic rather than induced by Arena's adaptive sampling.

\subsection{Protocol Effects: Arena versus MT-Bench}
\label{sec:mtbench}
The preceding analyses characterize feasibility limits under open-ended evaluation with user-generated prompts (e.g., Chatbot Arena). We now contrast this regime with MT-Bench, a curated benchmark that evaluates models on a fixed set of prompts with structured tasks and explicit rubrics. This comparison isolates how evaluation protocol design affects the distribution of observable preference margins and, consequently, statistical detectability.

\textbf{Comparing margin distributions across protocols.} MT-Bench exhibits preference margins shifted away from zero relative to Arena. As shown in Figure~\ref{fig:arena_mt}, its margin distribution places less mass in the near-tie regime and more mass at larger values. This shift reflects differences in per-judgment information induced by the evaluation protocol, not larger annotation budgets.

At fixed evaluation budgets ($n=50,100$), MT-Bench exhibits higher detectability distributions than Arena, at both the median and lower tail (Appendix~\ref{app:robustness_same_dudget}, Table~\ref{tab:same_budget_detectability}). Table~\ref{tab:correlation_sensitivity} quantifies how these feasibility limits are affected by intra-cluster correlation, reporting the inflation in required evaluation budgets under both median and near-tie regimes.

\begin{table}[t]
\centering
\caption{Sensitivity to correlation. Required decisive human judgments for 90\% detection power under intra-cluster correlation $\rho$. Inflation factor is the ratio of required nominal to effective sample sizes. ``---'' indicates $n_0\rho \geq 0.95$, rendering target power unattainable.}
\label{tab:correlation_sensitivity}
\resizebox{\columnwidth}{!}{
\begin{tabular}{lcccc}
\toprule
Regime & $\rho$ & $n(\rho{=}0)$ & Infl.$\times$ & $n_{\text{infl.}}$ \\
\midrule
Arena median        & 0.0000 &   26 &  1.000 &    26 \\
Arena median        & 0.0001 &   26 &  1.002 &    26 \\
Arena median        & 0.0010 &   26 &  1.025 &    26 \\
Arena median        & 0.0100 &   26 &  1.329 &    34 \\
\midrule
Hard near-tie ($\delta{=}0.05$) & 0.0000 & 1051 &  1.000 &  1051 \\
Hard near-tie ($\delta{=}0.05$) & 0.0001 & 1051 &  1.117 &  1174 \\
Hard near-tie ($\delta{=}0.05$) & 0.0010 & 1051 &  ---   & ${>}10000$ \\
Hard near-tie ($\delta{=}0.05$) & 0.0100 & 1051 &  ---   & ${>}10000$ \\
\bottomrule
\end{tabular}
}
\end{table}

\begin{figure}[t]
  \centering
  \includegraphics[width=\linewidth]{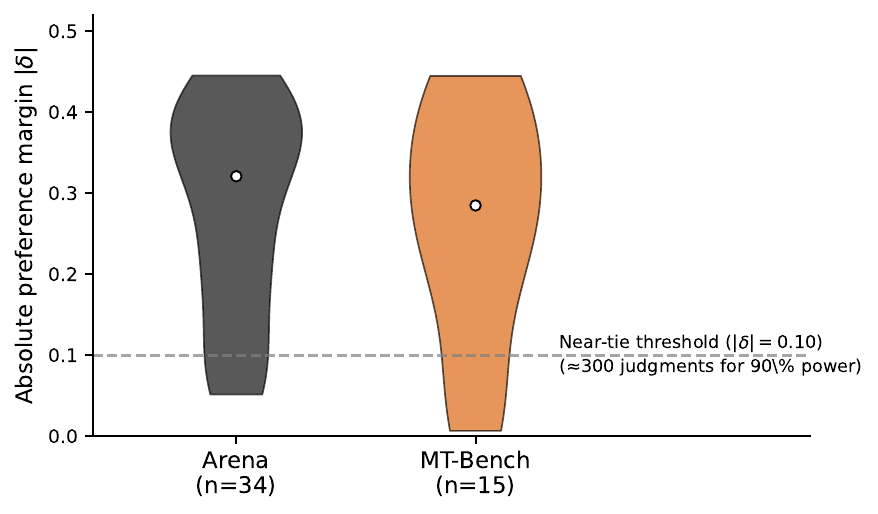}
    \caption{Distribution of absolute preference margins $|\delta|$ under Chatbot Arena and MT-Bench. Violin plots show margins for all model pairs meeting the sampling criteria of each protocol (Arena: $n=34$ pairs with $\geq$200 judgments; MT-Bench: $n=15$ pairs). The dashed line marks the near-tie threshold ($|\delta| = 0.10$), below which reliable detection typically requires hundreds of human judgments ($\approx$300 for 90\% power). MT-Bench exhibits a shift toward larger preference margins, concentrating more mass in higher-signal regimes associated with lower evaluation budgets.
}
  \label{fig:arena_mt}
\end{figure}

\textbf{Prompt-level variance as a mechanism.}
To understand why margins differ across protocols, we analyze prompt-level variability within the same model pairs. For each aligned pair evaluated under both Arena and MT-Bench, we compute the per-prompt win rate and measure its variance across prompts. This captures prompt-induced noise in the evaluation signal.

Across aligned model pairs, prompt-level variance is consistently higher under Arena than under MT-Bench. For pairs with sufficient coverage under both protocols, Arena’s prompt-level variance exceeds that of MT-Bench by approximately 1.5$\times$ at the median, driven primarily by higher between-prompt variability (Table~\ref{tab:variance_decomposition}). This reduction indicates that prompt curation and task structure suppress prompt-induced noise, increasing per-judgment KL divergence and shifting evaluations into higher-signal regimes.

\begin{table}[ht]
\centering
\small
\caption{
Variance decomposition for aligned model pairs evaluated under Chatbot Arena and MT-Bench. Reported values are medians across aligned pairs. Between-prompt variance captures variability in preference outcomes across prompts for the same model pair, while within-prompt variance reflects binomial noise from finite judgments. MT-Bench substantially reduces between-prompt variance relative to Arena, indicating more consistent preferences across prompts. This reduces prompt-induced noise, increasing the effective KL budget per judgment.}
\label{tab:variance_decomposition}
\begin{tabular}{lcc}
\toprule
 & Arena & MT-Bench \\
\midrule
Between-prompt variance & 0.102 & 0.066 \\
Within-prompt variance  & 0.000 & 0.015 \\
\midrule
\multicolumn{1}{p{4.0cm}}{\raggedright
Median per-pair btwn.-prompt var.\ ratio (Arena / MT)}
 & \multicolumn{2}{c}{1.49$\times$} \\
\bottomrule
\end{tabular}
\end{table}

\textbf{Interpretation.} These results clarify that the larger margins observed under MT-Bench should not be interpreted as artificial inflation or systematic bias. Rather, MT-Bench operates at a different point in the trade-off between coverage and statistical separability. By constraining the prompt distribution to tasks where model behaviors diverge more consistently, MT-Bench reduces prompt-level noise and increases the effective information content of each judgment. In contrast, Arena emphasizes breadth and open-ended interaction, which broadens coverage but induces greater variability across prompts and compresses observable margins.

\textbf{Implication for feasibility.} From the perspective of Section~\ref{sec:kl_theory}, these protocol differences translate directly into differences in accumulated KL budget. MT-Bench concentrates evaluation effort on comparisons with higher per-sample information, making reliable detection feasible under moderate budgets. Arena, by contrast, expends substantial effort in low-signal regimes where detectability is intrinsically limited, even as nominal evaluation budgets grow. Thus, feasibility limits are shaped not only by annotation volume but also, critically, by evaluation protocol design.

\textbf{Replications and robustness.} We replicate our analysis on a static snapshot of Chatbot Arena preferences and on heterogeneous pairing data from Infinity Chats \citep{jiang2025artificial}. Across settings, we observe consistent qualitative behavior: small preference margins require prohibitively large budgets for reliable detection, while larger effect sizes are detectable with modest evaluation effort, confirming that the budget–effect size relationship is robust across datasets and evaluation protocols.

\section{Practical Implications for Evaluation Design}
\label{sec:implications}

Our theoretical and empirical results have direct implications for how human preference evaluations should be designed, interpreted, and budgeted in practice. The central lesson is that statistical detectability is governed by the perceptual effect size induced by an evaluation protocol, rather than by annotation volume alone.

\subsection{Effect Size, KL Budget, and Detectability}

Reliable detection is governed primarily by effect size. When preference margins are small, additional judgments yield diminishing returns, while larger margins become detectable with modest budgets. Accordingly, failure to reject $H_0$ should often be interpreted as underpowered evaluation rather than evidence of model equivalence.

\subsection{Pilot Studies as Effect Size Screening}

Pilot human evaluations are most informative when viewed as effect size screening tools rather than as definitive tests of model superiority. A small pilot can quickly indicate whether a comparison lies in a regime where reliable detection is feasible or in a low-signal regime that would require substantially larger budgets.

In practice, pilot studies are underpowered whenever the detectable margin implied by the observed data and budget exceeds the effect size of interest. Reporting detectable margins alongside win rates provides a clearer account of what conclusions an evaluation can and cannot support.

\subsection{Closed-Form Budgeting Guidelines: How Many Human Judgments Are Enough?}
\label{sec:budgeting}

We provide a simple closed-form approximation that practitioners can use to estimate how many human judgments are required for a pairwise preference evaluation to be statistically informative.

Consider a binary preference test with success probability
$p = \tfrac{1}{2} + \delta$, where $\delta$ denotes the underlying preference margin. Testing
$H_0: p = \tfrac{1}{2}$ against $H_1: p \neq \tfrac{1}{2}$ at significance level $\alpha$ and power $1-\beta$
requires approximately
\begin{equation}
n(\delta)
\;\approx\;
\frac{(z_{1-\alpha/2} + z_{1-\beta})^2}{4\,\delta^2},
\label{eq:closed_form_budget}
\end{equation}
decisive human judgments, where $z_q$ is the $q$-quantile of the standard normal distribution.

As an illustrative example for commonly used choices $\alpha = 0.05$ and $1-\beta = 0.9$, this simplifies numerically to
\begin{equation}
\label{eq:closed_form_simplified}
n(\delta) \;\approx\; \frac{2.63}{\delta^2}.
\end{equation}

This reveals the quadratic dependence of sample complexity on preference margin. Applying ~\eqref{eq:closed_form_budget} explains why open-ended evaluations (e.g. Chatbot Arena) often operate in infeasible low-signal regimes, while curated benchmarks (e.g. MT-Bench) remain detectable under moderate budgets.

In practice, a small pilot evaluation can be used to estimate $\hat{\delta}$ and assess feasibility via \eqref{eq:closed_form_budget}. When the implied budget exceeds available resources, non-detection should be interpreted as underpowered rather than as evidence of no improvement. Broader implications for benchmark design are discussed in Appendix~\ref{app:benchmark_design}, with additional guidance on extrapolating effect sizes and transferring margins across protocols in Appendix~\ref{app:benchmark_budget_planning}. The limits of allocation-based gains are summarized in Section~\ref{sec:protocol-optimality} and Appendix~\ref{app:algorithm1-empirical}

\section{Related Work}

Our work relates to human evaluation of generative models, benchmark design, statistical power analysis, and preference-based learning.

\textbf{Human evaluation and benchmark design.} Human judgments are widely used to evaluate generative systems when automated metrics are misaligned with subjective quality \cite{gehrmann2021gem,elangovan2024considers}. For large language models, pairwise human preferences have emerged as a dominant evaluation primitive, exemplified by Chatbot Arena \cite{chiang2024chatbot} and MT-Bench \cite{zheng2023judging}. Prior work has emphasized leaderboard construction, aggregation, inter-annotator agreement, and the sensitivity of evaluation outcomes to benchmark design, prompt selection, and task framing \cite{dubois2023alpacafarm,gehrmann2021gem}. We complement this literature by showing that evaluation design also determines \emph{statistical detectability}: prompt distributions that amplify qualitative differences induce larger preference margins and reduce required evaluation budgets, while open-ended evaluations compress effect sizes and inflate sample complexity.

\textbf{Statistical power and sample complexity.} While our proofs rely on classical hypothesis testing tools \citep{lehmann2005testing,casella2002statistical}, our contribution is a minimax characterization of evaluation feasibility under adversarial prompt heterogeneity and an empirical analysis of where LLM evaluation protocols fall relative to this boundary. Prior work on adaptive sampling \citep{jamieson2014best, anonymous2025fewer} has shown nonuniform allocation can yield gains when signal varies across items; our negative result (Theorem~\ref{thm:minimax-allocation}) identifies the diffuse-signal regime where such gains are provably unattainable.

\textbf{Preference learning and efficient evaluation.} Pairwise preferences are foundational in learning-to-rank, reinforcement learning from human feedback, and social choice \cite{christiano2017deep,rafailov2023direct, bradley1952rank}. While much of this literature focuses on learning policies or global rankings, we instead treat preferences as an evaluation signal and study the detectability of improvements given finite data. Recent work on adaptive testing, example selection, and automated evaluators \cite{liu2023geval} is complementary to our analysis: such methods may reduce constant factors but do not eliminate the fundamental dependence of detectability on perceptual effect size that we characterize.

\section{Limitations}
\label{sec:limitations}
Our analysis makes several simplifying assumptions that bound the scope of our conclusions. We treat human judgments as independent and identically distributed; as shown empirically in Section~\ref{sec:arena-independence}, such correlations are negligible for well-sampled model pairs in our datasets and would only further tighten the best-case feasibility bounds if present. Our estimates should be interpreted as best-case benchmarks under idealized assumptions.

Second, our empirical analysis focuses on chat-oriented large language models, with extensions to image preference data and code generation; preference margins may differ in other domains such as mathematical reasoning or specialized technical tasks. Extending the analysis to additional modalities and task types is an important future direction.

Finally, the Arena data reflects a dynamic evaluation environment in which models improve over time and evaluation effort is not uniformly allocated. As discussed in Section~\ref{sec:robustness}, this induces selection effects that shape which comparisons become well-sampled, without systematically biasing them toward low-signal regimes. Our conclusions characterize the detectability of comparisons that attract sustained evaluation attention, rather than all possible model comparisons.

\section{Conclusion}
\label{sec:conclusion}

We studied when human preference evaluations can reliably detect improvements in generative models. We showed that detectability is governed by the total KL budget accumulated by an evaluation protocol, yielding sharp feasibility limits under heterogeneous prompts. Empirical analysis of open-ended evaluations reveals that many well-sampled model comparisons operate in low-signal regimes where reliable detection would require substantially larger budgets than those typically collected.

These results suggest that negative or inconclusive outcomes should often be interpreted as underpowered rather than as evidence of model equivalence. \textbf{Human evaluation results that do not explicitly justify detectability (via pilot estimation, power analysis, or an equivalent feasibility argument) should not be treated as reliable evidence for or against model improvements.}

\section*{Impact Statement}

This work provides methodological guidance for designing and interpreting human evaluations of generative models under limited budgets by clarifying the relationship between perceptual effect size, sample complexity, and statistical reliability. By explaining why many contemporary evaluations are underpowered, our results can help practitioners plan evaluations more efficiently and avoid overconfident conclusions drawn from insufficiently powered evaluations. We emphasize that these findings identify statistical limits rather than arguing against human evaluation, and should be used to support more transparent and responsible evaluation practices.

\nocite{}

\bibliography{main}
\bibliographystyle{icml}

\newpage
\appendix
\onecolumn
\appendix


\section{Sensitivity to Correlated Judgments}
\label{app:correlation}

To account for intra-cluster correlation, we compute the required nominal sample size to achieve an effective sample size of $n_0$ using the exact relationship derived from the standard ICC variance formula.

For $n$ observations from a single cluster with intra-cluster correlation $\rho$, the variance of the sample mean is $\mathrm{Var}(\bar{Y}) = (\sigma^2/n)[1 + (n-1)\rho]$. Because prompt reuse in Arena is negligible (Appendix~\ref{app:robustness_prompt_resue}), treating all judgments as a single exchangeable cluster is conservative; modeling many small clusters would further reduce effective correlation. Setting this equal to the variance of $n_0$ independent observations and solving for the required nominal sample size yields:
\begin{equation}
\label{eq:ninfl}
n_{\text{infl}} = \frac{n_0(1 - \rho)}{1 - n_0\rho}.
\end{equation}

\textbf{Derivation:} For effective sample size $n_0$, we require $\sigma^2/n_{\text{infl}} \cdot [1 + (n_{\text{infl}} - 1)\rho] = \sigma^2/n_0$. Expanding: $n_0[1 + n_{\text{infl}}\rho - \rho] = n_{\text{infl}}$, which rearranges to $n_0(1-\rho) = n_{\text{infl}}(1 - n_0\rho)$.

When $n_0\rho \geq 0.95$, the denominator $(1 - n_0\rho) \leq 0.05$ becomes small and $n_{\text{infl}}$ exceeds 10,000; we report such cases as ``$>10000$''. This threshold reflects infeasibility: as $n_0\rho \to 1$, achieving effective sample size $n_0$ requires infinite observations because judgments become nearly redundant.

Table~\ref{tab:correlation_sensitivity} quantifies how ICC inflates required sample sizes under both median and near-tie regimes. The inflation factor is $n_{\text{infl}}/n_0 = (1-\rho)/(1-n_0\rho)$.

\textbf{Empirical validation:} Cluster-robust standard errors (allowing arbitrary within-prompt correlation) yield negligible variance inflation across all Arena pairs (median ratio 1.000; 99th percentile $<$ 1.001), confirming that detectability limits arise from small margins rather than correlation. Any nonzero intra-cluster correlation therefore strictly reduces the effective KL budget accumulated per judgment, reinforcing the best-case feasibility limits established in Section~3 rather than undermining them.

\section{Proofs for Section~\ref{sec:kl_theory}}
\label{app:kl_proofs}

\subsection{Preliminaries}
Under \eqref{eq:kl_model}, conditional on a fixed $\delta_{1:B}$ the joint likelihood factorizes:
\[
P_{\delta}(Y_{1:B})=\prod_{t=1}^B p_t^{Y_t}(1-p_t)^{1-Y_t},
\qquad p_t=\tfrac12+\delta_t,
\]
and under $H_0$, $p_t=\tfrac12$ for all $t$. Hence
\[
\mathrm{KL}(P_{\delta}\|P_0)=\sum_{t=1}^B \mathrm{KL}(\mathrm{Bern}(p_t)\|\mathrm{Bern}(1/2)).
\]
Define the log-likelihood ratio
\[
L(Y):=\log\frac{dP_{\delta}}{dP_0}(Y_{1:B})
=\sum_{t=1}^B \ell_t(Y_t),
\quad
\ell_t(y)= y\log\frac{p_t}{1/2}+(1-y)\log\frac{1-p_t}{1/2}.
\]
Note that $\mathbb{E}_{\delta}[L]=\mathrm{KL}(P_{\delta}\|P_0)$ and $\mathbb{E}_0[e^{L}]=1$.

\subsection{Proof of Theorem~\ref{thm:kl_lower_main}}
\begin{proof}
Fix any alternative $\delta$ with $\mathrm{KL}(P_{\delta}\|P_0)\ge \mathsf{K}$.
By the Bretagnolle--Huber inequality (a standard consequence of the change-of-measure identity),
for any test $\phi$,
\[
\Pr_0(\phi=1)+\Pr_{\delta}(\phi=0)\ \ge\ \tfrac12 \exp\!\big(-\mathrm{KL}(P_{\delta}\|P_0)\big).
\]
Since $\mathrm{KL}(P_{\delta}\|P_0)\ge \mathsf{K}$, we obtain
\[
\Pr_0(\phi=1)+\Pr_{\delta}(\phi=0)\ \ge\ \tfrac12 e^{-\mathsf{K}}.
\]
Taking the supremum over all $\delta$ with KL budget at least $\mathsf{K}$ yields the claim.
Rearranging $\tfrac12 e^{-\mathsf{K}}\le \alpha+\beta$ gives the necessity condition
$\mathsf{K}\gtrsim \log\frac{1}{\alpha+\beta}$.
\end{proof}

\subsection{A concentration lemma for bounded log-likelihood increments}
\begin{lemma}[LLR concentration under $P_{\delta}$]
\label{lem:llr_conc_kl}
Assume $|\delta_t|\le 1/4$ for all $t$. Then $\ell_t(Y_t)$ are independent and uniformly bounded.
Moreover, there exist universal constants $a_0,c_0>0$ such that for any fixed $\delta$,
\[
\Pr_{\delta}\!\left(L \le \tfrac12\,\mathbb{E}_{\delta}[L]\right)
\le \exp\!\left(-c_0\,\mathbb{E}_{\delta}[L]\right)
\qquad \text{whenever } \mathbb{E}_{\delta}[L]\ge a_0.
\]
\end{lemma}

\begin{proof}
Boundedness: for $|\delta_t|\le 1/4$, we have $p_t\in[1/4,3/4]$, hence
\[
\left|\log\frac{p_t}{1/2}\right|\le \log(3/2),
\qquad
\left|\log\frac{1-p_t}{1/2}\right|\le \log(3/2),
\]
so $|\ell_t(Y_t)|\le \log(3/2)$ almost surely.

Let $Z_t := \ell_t(Y_t)-\mathbb{E}_{\delta}[\ell_t(Y_t)]$, so $\sum_t Z_t = L-\mathbb{E}_{\delta}[L]$. By boundedness, $Z_t$ are independent, centered, and sub-Gaussian with a universal parameter.
Applying Bernstein's inequality yields
\[
\Pr_{\delta}\!\left(L-\mathbb{E}_{\delta}[L]\le -\tfrac12 \mathbb{E}_{\delta}[L]\right)
\le \exp\!\left(-\frac{c\,\mathbb{E}_{\delta}[L]^2}{\sum_t \mathrm{Var}(Z_t) + \mathbb{E}_{\delta}[L]}\right)
\]
for a universal constant $c>0$.
In this Bernoulli-shift regime, $\mathrm{Var}(\ell_t(Y_t))$ is bounded by a universal constant multiple of $\mathbb{E}_{\delta}[\ell_t(Y_t)]$
(since $\ell_t$ is bounded and the mean KL controls the variance up to constants on $p_t\in[1/4,3/4]$),
implying $\sum_t \mathrm{Var}(Z_t)\le C_1\,\mathbb{E}_{\delta}[L]$ for universal $C_1$.
Substituting gives
\[
\Pr_{\delta}\!\left(L\le \tfrac12\mathbb{E}_{\delta}[L]\right)
\le \exp(-c_0\,\mathbb{E}_{\delta}[L])
\]
once $\mathbb{E}_{\delta}[L]\ge a_0$ for a suitable universal $a_0$.
\end{proof}

\subsection{Proof of Theorem~\ref{thm:kl_upper_main}}
\label{app:proof_33}
\begin{proof}
Fix a target $\mathsf{K}>0$ and define the test
\[
\phi_{\mathsf{K}}(Y):=\mathbf{1}\left\{L(Y)\ge \tfrac12 \mathsf{K}\right\}.
\]
(Type-I control) Under $P_0$, $e^{L}$ is the likelihood ratio and satisfies $\mathbb{E}_0[e^{L}]=1$.
By Markov's inequality,
\[
\Pr_0\!\left(L\ge \tfrac12\mathsf{K}\right)
=\Pr_0\!\left(e^{L}\ge e^{\mathsf{K}/2}\right)
\le e^{-\mathsf{K}/2}.
\]
(Type-II control) Consider any alternative $\delta$ such that $\mathrm{KL}(P_{\delta}\|P_0)=\mathbb{E}_{\delta}[L]\ge C\mathsf{K}$.
By Lemma~\ref{lem:llr_conc_kl},
\[
\Pr_{\delta}\!\left(L<\tfrac12\mathsf{K}\right)
\le \Pr_{\delta}\!\left(L\le \tfrac12 \mathbb{E}_{\delta}[L]\right)
\le \exp(-c_0\,\mathbb{E}_{\delta}[L])
\le \exp(-c_0 C \mathsf{K}),
\]
provided $C\mathsf{K}\ge a_0$. Choosing $C$ large enough and absorbing constants yields
\[
\Pr_0(\phi_{\mathsf{K}}=1)+\sup_{\mathrm{KL}\ge C\mathsf{K}}\Pr_{\delta}(\phi_{\mathsf{K}}=0)
\le e^{-\mathsf{K}/2}+e^{-c\mathsf{K}}
\le e^{-c'\mathsf{K}}
\]
for universal constants $c,c'>0$ and all $\mathsf{K}$ above a universal constant.
This proves the claim.
\end{proof}

\subsection{Proofs for Section~\ref{sec:protocol-optimality} (Allocation Optimality)}
\label{app:proofs_34}
This subsection provides proof sketches for Theorems~\ref{thm:minimax-allocation} and~\ref{thm:stochastic-adaptive}. Throughout, we reuse the KL divergence bounds from Lemma~3.1 and standard hypothesis testing inequalities (e.g., Bretagnolle--Huber). Constants are omitted when immaterial.

\subsubsection{Proof of Theorem~\ref{thm:minimax-allocation} (Minimax Optimality of Proportional Allocation)}

\textbf{Least-favorable alternative.}
Fix a (possibly adaptive) allocation policy $\pi$ producing counts
$N_1,\dots,N_m$ with $\sum_i N_i = B$. Let
\[
i^\star \in \arg\min_{i\in[m]} \frac{N_i}{w_i}.
\]

\begin{lemma}[Least-favorable alternative]
\label{lem:least-favorable}
For any allocation $(N_i)$ and signal budget $\mu_2>0$, consider the randomized
alternative $\delta^{(i^\star)}$ defined by
\[
\delta_i =
\begin{cases}
+ \sqrt{\mu_2 / w_{i^\star}} & \text{with probability } \tfrac12,\ i=i^\star,\\
- \sqrt{\mu_2 / w_{i^\star}} & \text{with probability } \tfrac12,\ i=i^\star,\\
0 & \text{otherwise}.
\end{cases}
\]
Then $\delta^{(i^\star)} \in \mathcal H_1(\mu_2)$ and, among all alternatives in
$\mathcal H_1(\mu_2)$, it minimizes the KL divergence accumulated under the
allocation $(N_i)$ up to universal constants.
\end{lemma}

\textbf{Proof of Lemma~\ref{lem:least-favorable}.}
First, verify that $\delta^{(i^\star)} \in \mathcal H_1(\mu_2)$:
\[
\sum_{i=1}^m w_i \delta_i^2 = w_{i^\star} \cdot \left(\sqrt{\mu_2/w_{i^\star}}\right)^2 = w_{i^\star} \cdot \frac{\mu_2}{w_{i^\star}} = \mu_2. \quad \checkmark
\]

For our construction $\delta^{(i^\star)}$, applying Lemma~3.1 gives:
\[
\mathrm{KL}\!\left(P_{\delta^{(i^\star)}} \,\middle\|\, P_0\right)
\;\lesssim\;
N_{i^\star} \delta_{i^\star}^2
\;=\;
N_{i^\star} \cdot \frac{\mu_2}{w_{i^\star}}
\;=\;
\mu_2 \cdot \frac{N_{i^\star}}{w_{i^\star}}.
\]

\textbf{Why is this least-favorable?} Consider any alternative $\tilde\delta$ with $\sum_i w_i \tilde\delta_i^2 = \mu_2$ (we can restrict to equality by monotonicity). The KL divergence under allocation $(N_i)$ satisfies:
\[
\mathrm{KL}\!\left(P_{\tilde\delta} \,\middle\|\, P_0\right)
\;\lesssim\;
\sum_{i=1}^m N_i \tilde\delta_i^2.
\]

For any distribution $(\tilde\delta_i^2)$ satisfying the constraint:
\begin{align*}
\sum_{i=1}^m N_i \tilde\delta_i^2 
&= \sum_{i=1}^m w_i \tilde\delta_i^2 \cdot \frac{N_i}{w_i} \\
&\geq \left(\min_{i} \frac{N_i}{w_i}\right) \sum_{i=1}^m w_i \tilde\delta_i^2 \\
&= \frac{N_{i^\star}}{w_{i^\star}} \cdot \mu_2,
\end{align*}
with equality when all mass is concentrated on $i^\star$. This proves that $\delta^{(i^\star)}$ minimizes the KL divergence among all alternatives satisfying the constraint.
\hfill $\square$

\textbf{Intuition.}
Any alternative satisfying $\sum_i w_i \delta_i^2 \ge \mu_2$ must place at least
$\mu_2$ total $\delta^2$-mass across prompt types. Concentrating this mass on the
type with smallest normalized allocation $N_i/w_i$ minimizes the resulting KL
divergence, since KL scales linearly in both $N_i$ and $\delta_i^2$
(Lemma~3.1). The adversary always attacks the weakest link.

\textbf{Completing the proof of Theorem~\ref{thm:minimax-allocation}.}
Fix an arbitrary (possibly adaptive) allocation policy $\pi$ producing counts
$N_1,\dots,N_m$ with $\sum_i N_i = B$. Define the allocation bottleneck:
\[
\Lambda(\pi) := \mathbb{E}_\pi\left[\min_{i \in [m]} \frac{N_i}{w_i}\right].
\]

From Lemma~\ref{lem:least-favorable}, for any realized allocation, the adversary can choose $\delta^{(i^\star)}$ satisfying:
\[
\mathrm{KL}\!\left(P_{\delta^{(i^\star)}}^\pi \,\middle\|\, P_0^\pi\right)
\;\le\;
c\,\mu_2 \cdot \min_{i} \frac{N_i}{w_i}.
\]

Taking expectation over the randomness of $\pi$ yields the bound in terms of
$\Lambda(\pi) = \mathbb{E}_\pi[\min_i N_i / w_i]$.

Applying the Bretagnolle--Huber inequality then gives, for any test $\varphi$,
\[
\Pr_0^\pi(\varphi=1) + \Pr_{\delta^{(i^\star)}}^\pi(\varphi=0)
\;\ge\;
\tfrac12 \exp\!\big(-C\,\mu_2\,\Lambda(\pi)\big).
\]

Since the adversary may choose $\delta$ after observing the realized allocation,
this lower bound holds for all adaptive policies $\pi$. This gives the minimax lower bound:
\[
\inf_{\pi, \varphi} R(\pi, \varphi) \geq \frac{1}{2} \exp(-C' B \mu_2),
\]
where we used that $\Lambda(\pi) \leq B$ for any allocation.

Finally, for proportional allocation with $N_i \approx B w_i$, we have
$\Lambda(\pi) \ge c' B$, which combined with Theorem~\ref{thm:minimax-allocation} yields the stated minimax rate.
\hfill $\square$

\subsubsection{Proof Sketch of Theorem~\ref{thm:stochastic-adaptive} (Adaptive Gains under Stochastic Heterogeneity)}

We provide a proof sketch highlighting the main steps and rate dependencies. A fully detailed finite-sample proof follows by combining uniform Hoeffding concentration with the likelihood-ratio testing argument of Appendix~\ref{app:proof_33}, and is omitted for brevity. We outline the argument for the two-stage screening allocation.

\textbf{Step 1: Screening accuracy.}
In Stage~1, each prompt type is sampled $b$ times. By Hoeffding's inequality,
for $b = \Omega(\log m)$, the empirical estimates $\hat\delta_i^2$ uniformly
concentrate around $\delta_i^2$ with probability at least $1-\eta$.
Consequently, the set $\hat S_q$ of the top $\lceil qm\rceil$ empirical values
coincides with the oracle set $S_q(\delta)$ except on an event of probability
at most $\eta$.

\textbf{Step 2: KL budget under correct screening.}
Conditioned on correct screening, all Stage~2 samples are drawn from prompt
types in $S_q(\delta)$. Applying Lemma~3.1 and linearity of KL,
the accumulated KL divergence in Stage~2 satisfies
\[
\mathrm{KL}(P_\delta^\pi \,\|\, P_0^\pi)
\;\ge\;
c\, B_2\, \mu_{2,q}(\delta),
\]
where $B_2$ is the remaining budget after screening.

\textbf{Step 3: Error decomposition.}
Decomposing on the screening event,
\[
\Pr(\text{error})
\;\le\;
\Pr(\text{screening fails})
+
\Pr(\text{test fails} \mid \text{screening succeeds})
\;\le\;
\eta + \exp(-c B_2 \mu_{2,q}(\delta)).
\]
Taking expectation over $\delta\sim D^m$ yields the stated bound. The comparison
to proportional allocation follows directly from Proposition~3.4.

\textbf{Rates.}
In particular, screening succeeds with probability at least
$1 - 2\exp(-c_1 b)$ for a universal constant $c_1>0$, and conditional on correct
screening, the Stage~2 test achieves total error at most
$\exp(-c_2 B_2 \mu_{2,q})$ for a universal constant $c_2>0$.
Thus, choosing $b = \Theta(\log m)$ ensures that the screening overhead is
negligible relative to the exponential decay governed by $B_2\mu_{2,q}$.

\subsubsection{Two-Stage Screening Allocation Procedure}

\begin{algorithm}[H]
\caption{Two-Stage Screening Allocation}
\label{alg:two-stage}
\begin{algorithmic}
\STATE {\bfseries Input:} total budget $B$, screening budget per type $b$, retention fraction $q$
\STATE {\bfseries Output:} decision on whether a detectable preference difference exists
\STATE
\STATE {\bfseries Stage 1 (Screening):}
\FOR{$i = 1$ {\bfseries to} $m$}
  \STATE Collect $b$ judgments for prompt type $i$
  \STATE Compute empirical estimate $\hat{\delta}_i$
\ENDFOR
\STATE Let $\hat S_q$ be the $\lceil qm\rceil$ prompt types with largest $\hat{\delta}_i^2$
\STATE
\STATE {\bfseries Stage 2 (Focusing):}
\STATE Allocate remaining budget $B - mb$ by sampling only prompt types in $\hat S_q$
\STATE \hspace{1em} proportionally to their weights $w_i$
\STATE
\STATE {\bfseries Testing:}
\STATE Apply a two-sided binomial or likelihood-ratio test using Stage~2 outcomes
\end{algorithmic}
\end{algorithm}

\section{Implications for Benchmark Design and Budget Planning}
\label{app:implications}

\subsection{Empirical Verification of Two-Stage Allocation}
\label{app:algorithm1-empirical}

We empirically verify the behavior of Algorithm~\ref{alg:two-stage} using a controlled simulation designed to isolate the role of prompt-type heterogeneity.

\textbf{Simulation setup.}
We consider $m$ prompt types with weights $w \in \Delta_m$. For each type $i$, judgments are drawn i.i.d.\ as
\[
Y \sim \mathrm{Bernoulli}\!\left(\tfrac12 + \delta_i\right),
\]
with total evaluation budget $B$. We compare two allocation strategies:
(i) \emph{proportional allocation}, assigning $B w_i$ samples to each type, and
(ii) the two-stage screening-and-focusing strategy of Algorithm~\ref{alg:two-stage}, using a screening budget $b = \lceil c \log m \rceil$ per type and retaining a fraction $q$ of prompt types for stage~2. Hypothesis tests are two-sided binomial tests at level $\alpha = 0.05$, applied to stage-2 samples only.

\textbf{Signal regimes.}
We consider two regimes.
In the \emph{concentrated} regime, signal is unevenly distributed: a small subset of prompt types accounts for a disproportionate share of $\sum_i w_i \delta_i^2$.
In the \emph{diffuse} regime, signal is approximately uniform across types.
These settings correspond respectively to the assumptions of Theorems~3.6 and~3.5.

\textbf{Detectability under budget constraints.}
Figure~\ref{fig:algorithm1-power} reports empirical power as a function of total budget $B$.
In the concentrated regime, two-stage allocation consistently achieves higher power than proportional allocation at the same budget, validating the improved detectability predicted by Theorem~\ref{thm:stochastic-adaptive}. In contrast, in the diffuse regime, proportional allocation dominates; two-stage allocation underperforms due to screening overhead, consistent with the minimax optimality result of Theorem~\ref{thm:minimax-allocation}.

\textbf{Screening accuracy is not the mechanism.}
Figure~\ref{fig:algorithm1-screening} reports the Jaccard overlap between the selected set of prompt types and the oracle top-$q$ set ranked by $w_i \delta_i^2$. Even in the concentrated regime, exact recovery of the oracle set is limited and noisy. This demonstrates that the gains of Algorithm~\ref{alg:two-stage} do not rely on accurate identification of individual high-signal prompt types.

\textbf{Signal-mass concentration explains gains.}
Figure~\ref{fig:algorithm1-signal} reports the fraction of total signal mass captured by the retained prompt types,
\[
\frac{\sum_{i \in \hat S_q} w_i \delta_i^2}{\sum_{i=1}^m w_i \delta_i^2}.
\]
In the diffuse regime, this quantity concentrates at $q$, indicating that screening is no better than random focusing. In contrast, in the concentrated regime, Algorithm~\ref{alg:two-stage} captures substantially more than a $q$-fraction of total signal mass, despite imperfect set recovery. This concentration of signal mass directly explains the power improvements observed in Figure~\ref{fig:algorithm1-power} and aligns with the KL-based analysis underlying Theorem~\ref{thm:stochastic-adaptive}.

\begin{figure}[t]
    \centering
    \includegraphics[width=\linewidth]{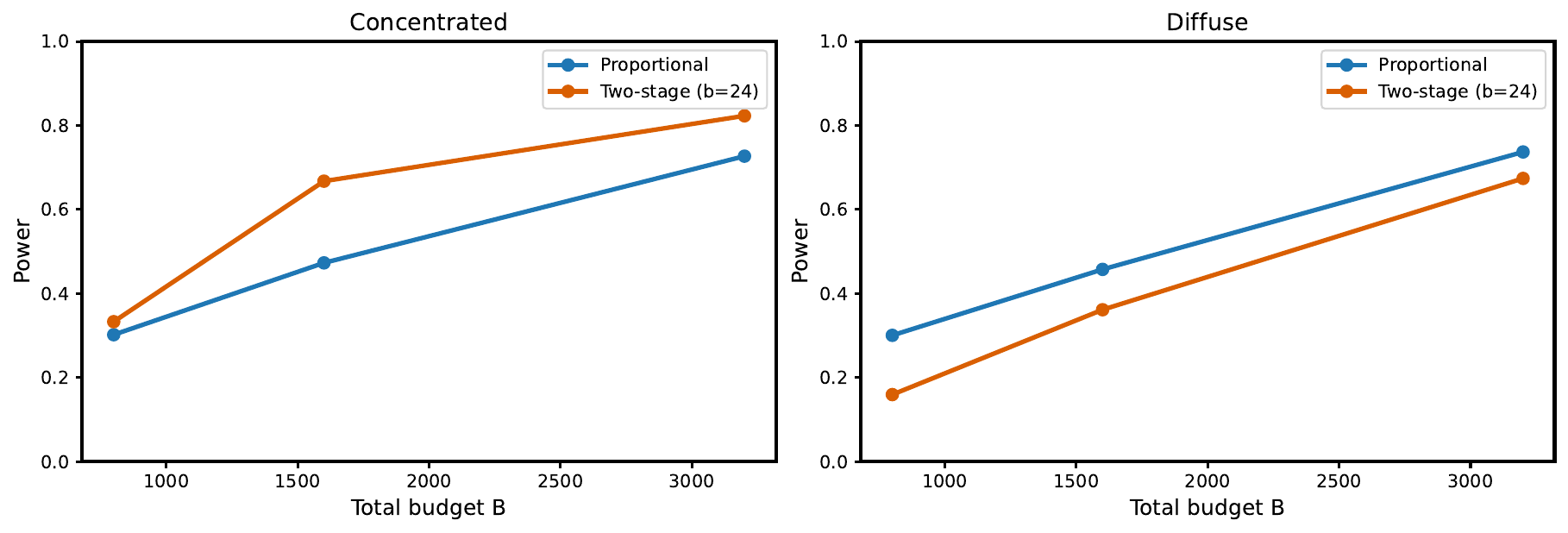}
    \caption{
    Empirical power as a function of total evaluation budget $B$. Two-stage allocation improves detectability when signal is concentrated across prompt types (left), but underperforms proportional allocation in the diffuse regime (right).
    }
    \label{fig:algorithm1-power}
\end{figure}

\begin{figure}[t]
    \centering
    \includegraphics[width=\linewidth]{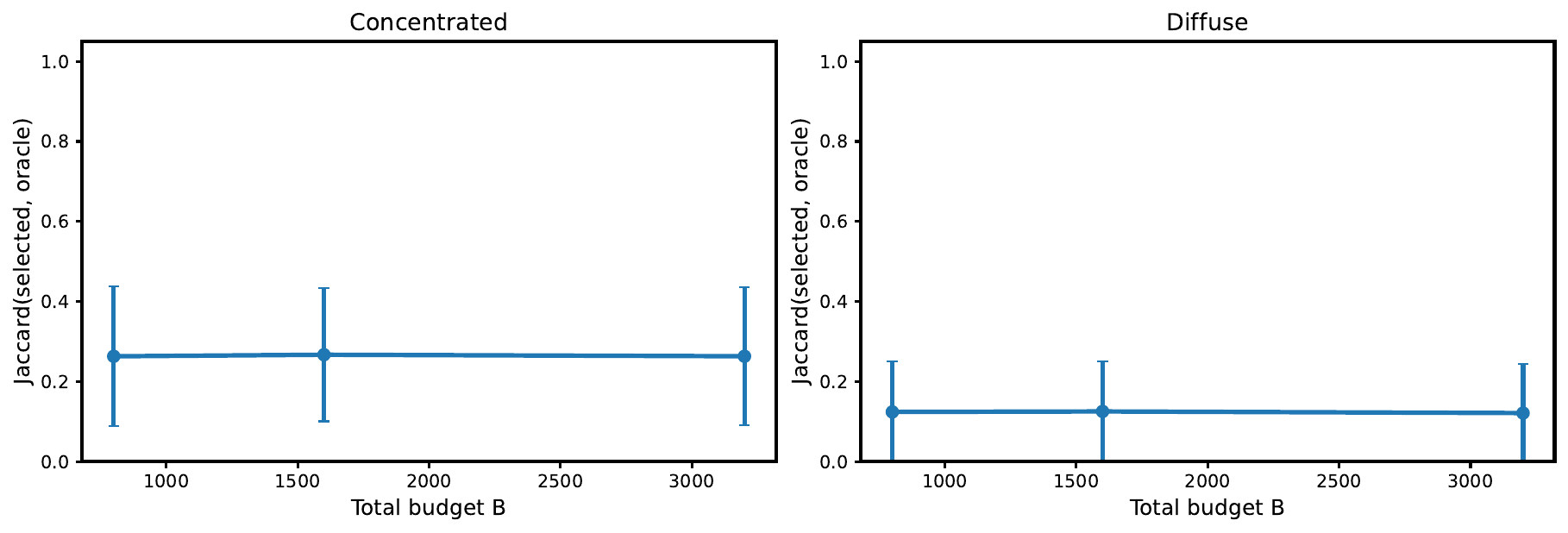}
    \caption{
    Screening accuracy measured by Jaccard overlap with the oracle top-$q$ set.
    Exact recovery of high-signal prompt types is limited even in the concentrated regime,
    indicating that screening accuracy alone does not explain the gains of Algorithm~\ref{alg:two-stage}.
    }
    \label{fig:algorithm1-screening}
\end{figure}

\begin{figure}[t]
    \centering
    \includegraphics[width=\linewidth]{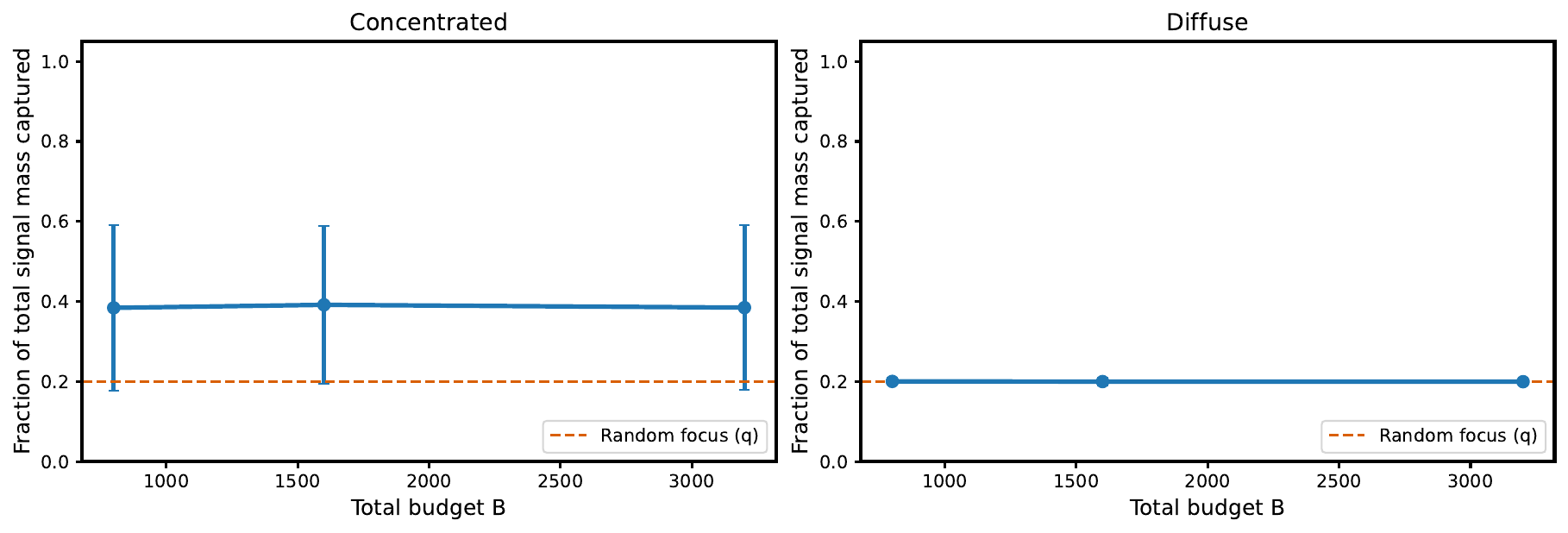}
    \caption{
    Fraction of total signal mass captured by the retained prompt types.
    The dashed line indicates the $q$ baseline corresponding to random focusing.
    Algorithm~\ref{alg:two-stage} concentrates substantially more signal mass than random selection in the concentrated regime,
    while no such concentration is possible in the diffuse regime.
    }
    \label{fig:algorithm1-signal}
\end{figure}

\subsection{Offline Replay Validation of Algorithm~\ref{alg:two-stage}}
\label{app:offline-replay}

We next test whether the same qualitative behavior appears under real human-evaluation data constraints.

We further validate Algorithm~\ref{alg:two-stage} via \emph{offline replay} on existing human judgment datasets. Because Arena and MT-Bench do not support online adaptive allocation, we evaluate counterfactual strategies via offline replay. We resample without replacement from fixed pools of judgments so all strategies share the same information and budget constraints

\textbf{Replay protocol.}
For each model pair with sufficient human judgments, we compare proportional allocation to the two-stage screening-and-focusing strategy of Algorithm~\ref{alg:two-stage} under a fixed budget~$B$. In each replay trial, judgments are sampled without replacement from the existing pool. Detection is assessed using the same two-sided hypothesis test as in Section~\ref{sec:method}. Reported power corresponds to the fraction of trials in which the null hypothesis is rejected. All results are averaged over repeated replay trials.

\textbf{Chatbot Arena.}
For Chatbot Arena, each question typically appears only once per model pair, making prompt-level screening ineffective without additional coarsening. We therefore partition prompts into coarse types (e.g., by language), ensuring that each type contains multiple judgments. Offline replay under realistic per-pair budgets shows that two-stage allocation does \emph{not} improve detection power on average, and in some cases reduces power due to screening noise overwhelming any potential concentration gains. This outcome reflects highly diffuse prompt-level signal in Arena and empirically confirms the negative regime of Theorem~\ref{thm:stochastic-adaptive}, in which adaptive allocation provides no benefit over proportional sampling.

\textbf{MT-Bench (human judgments).}
MT-Bench human evaluations contain multiple independent judgments per question, allowing each question to be treated directly as a prompt type. Because MT-Bench collects fewer total judgments per model pair, we evaluate Algorithm~\ref{alg:two-stage} using pair-specific budgets capped at a maximum value, ensuring that proportional and two-stage allocation operate under identical data constraints. Offline replay shows modest but systematic gains from two-stage allocation for a subset of model pairs, increasing mean detection power from $0.067$ under proportional allocation to $0.081$ under two-stage allocation across $15$ well-sampled pairs. Gains are concentrated in a small number of comparisons (maximum observed gain $0.175$), while most pairs show no improvement, indicating diffuse signal. While small in absolute magnitude, these gains occur in low-power regimes where detection is otherwise unlikely, making even modest improvements meaningful. This selective pattern of improvement is consistent with the stochastic-heterogeneity regime characterized in Theorem~\ref{thm:stochastic-adaptive}.

\textbf{Summary.}
To summarize, these replay experiments mirror the theoretical dichotomy of Section~\ref{sec:protocol-optimality}: adaptive allocation yields no benefit in diffuse regimes such as Chatbot Arena, but can provide measurable gains when prompt-level signal is sufficiently concentrated, as in MT-Bench. These results empirically validate both the positive and negative cases of Algorithm~\ref{alg:two-stage} without requiring new data collection or online deployment. Because replay conditions on a fixed pool of existing judgments, these results characterize detectability under fixed data availability rather than the behavior of a deployable online adaptive policy.

\subsection{Implications for Benchmark Design}
\label{app:benchmark_design}
Benchmark design implicitly selects an operating point on the trade-off between coverage and statistical separability. Protocols that emphasize open-ended interaction and prompt diversity broaden coverage but induce smaller preference margins, placing evaluations in low-signal regimes where detectability is intrinsically limited. In contrast, curated benchmarks concentrate evaluation effort on prompts that reduce variability and increase per-sample information. Our results suggest that these choices should be made explicitly: different benchmarks may be appropriate at different stages of model development, but detectability ultimately depends on the information induced by the prompt distribution rather than on evaluation scale alone.

\subsection{Using Existing Benchmarks for Budget Planning}
\label{app:benchmark_budget_planning}
Existing benchmarks can substitute for pilot studies when they reflect a comparable evaluation protocol and prompt distribution. In such cases, observed preference margins provide external estimates of effect size that can be used directly for budget planning.

Our results caution, however, that effect sizes measured under curated protocols may not transfer to open-ended evaluation settings (such as Chatbot Arena). Benchmarks such as MT-Bench systematically induce larger margins than benchmarks such as Chatbot Arena by reducing prompt-induced variability through curation and task structure, and thus overestimate detectability when applied to deployment-oriented comparisons. Practitioners should therefore plan budgets conservatively when extrapolating from curated benchmarks to open-ended evaluations.

\subsection{Implications of Allocation Optimality}
We summarize how the allocation results in Section~\ref{sec:protocol-optimality} translate into practical consequences. Under adversarial prompt heterogeneity, no adaptive allocation strategy improves the detectability exponent beyond constant factors. Proportional allocation is minimax-optimal.

When signal is unevenly concentrated across prompt types, adaptive screening and focusing increase the realized KL budget per judgment and yield strictly better error exponents.

\textbf{When two-stage allocation is worthwhile.}
Two-stage allocation is most beneficial when the top $q$ prompt types account for a disproportionate share of total signal energy. Concretely, we say that signal is concentrated if
\[
\sum_{i \in S_q} \hat{\delta}_i^2 \;\ge\; \kappa\, q \sum_{i=1}^m \hat{\delta}_i^2,
\] where $S_q$ indexes the $\lceil qm\rceil$ largest values of $\hat{\delta}_i^2$ and $\kappa > 1$ is a concentration threshold (we use $\kappa \in \{1.5,2\}$).

\section{Robustness Checks and Sensitivity Analyses}
\label{app:robustness}
\subsection{Prompt Reuse Frequency}
\label{app:robustness_prompt_resue}
We first examine prompt reuse at the level of individual model pairs. Using the Arena-provided question ID as a prompt identifier, we find that prompt reuse is essentially nonexistent among well-sampled comparisons. Across all model pairs with at least 200 decisive judgments, over 99.9\% of $(\text{pair}, \text{prompt})$ clusters have size one, with only a single cluster of size two observed in the entire dataset. Thus, almost every judgment for a given model pair corresponds to a unique prompt, precluding substantial prompt-induced correlation.

\subsection{Unique-Prompt Stress Test}
\label{app:robustness_unique_prompt}
As a further robustness check, we recompute preference margins after restricting the data to at most one judgment per $(\text{model pair}, \text{prompt})$, retaining only the first occurrence of each prompt. This restriction leaves the estimated margins virtually unchanged. Across well-sampled pairs, the median absolute change in the estimated margin is zero, and the maximum observed change is $7.3 \times 10^{-4}$, several orders of magnitude smaller than the margins governing detectability. In particular, the prevalence of near-tie comparisons is unaffected by this restriction.

\subsection{Stability of Preference Margins Under Time and Exposure}
Despite their relative rarity, near-tie comparisons are consequential because they dominate evaluation difficulty. Figure~\ref{fig:early_late_margins} shows detectability curves for representative low-margin comparisons obtained by subsampling human judgments and testing $H_0 : p = 0.5$ versus $H_1 : p \neq 0.5$ at $\alpha = 0.05$. In this regime, detectability increases slowly with evaluation budget: even hundreds of judgments are often insufficient to reliably detect a real improvement. For example, comparisons with estimated margins around $|\hat{p} - 0.5| \approx 0.04$ require on the order of $10^3$ independent judgments to achieve 90\% detection probability (via Eq.~\ref{eq:closed_form_budget}), far exceeding typical evaluation budgets.

\textbf{Equal-exposure cohort analysis.}
To further isolate stability from endogenous exposure effects, we condition on an equal-exposure cohort in which each model pair receives exactly 50 early judgments and reaches at least 200 total. Within this controlled cohort, early and final preference margins are strongly correlated (Spearman $\rho = 0.93$), with no evidence of drift toward lower separability as additional judgments accrue. Near-tie comparisons are rare under equal exposure, but all pairs with early $|\hat{p} - 0.5| < 0.10$ remain ambiguous after 200+ judgments, confirming that low-signal comparisons reflect intrinsic preference difficulty rather than artifacts of adaptive sampling.

\begin{figure}[ht]
    \centering
    \includegraphics[width=0.38\textwidth]{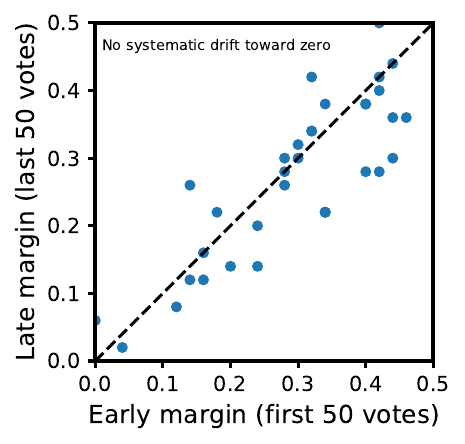}
    \caption{
    Early- versus late-vote preference margin estimates for well-sampled Chatbot Arena model pairs ($\ge200$ decisive judgments). Each point corresponds to a model pair, with margins estimated from the first 50 and last 50 judgments respectively. Margins are largely stable over time, with no systematic drift toward zero, indicating that additional votes primarily reduce estimation variance rather than revealing latent near-ties.}
    \label{fig:early_late_margins}
\end{figure}

\subsection{Uncertainty in Tail Quantiles}

To assess the stability of the margin quantiles reported in Table 1, we compute bootstrap 95\% confidence intervals by resampling model pairs with replacement. Table \ref{tab:tail_ci} reports the results.

\begin{table}[h]
\centering
\caption{Bootstrap 95\% confidence intervals for tail quantiles of $|\delta|$.}
\label{tab:tail_ci}
\begin{tabular}{lcccc}
\toprule
Protocol & Statistic & Point & CI (2.5\%) & CI (97.5\%) \\
\midrule
Arena & $p_{10}(|\delta|)$ & 0.082 & 0.057 & 0.186 \\
Arena & $p_{25}(|\delta|)$ & 0.190 & 0.088 & 0.264 \\
Arena & $p_{50}(|\delta|)$ & 0.321 & 0.216 & 0.371 \\
MT-Bench & $p_{10}(|\delta|)$ & 0.047 & 0.007 & 0.230 \\
MT-Bench & $p_{25}(|\delta|)$ & 0.187 & 0.037 & 0.285 \\
MT-Bench & $p_{50}(|\delta|)$ & 0.285 & 0.181 & 0.403 \\
\bottomrule
\end{tabular}
\end{table}

Uncertainty is largest for extreme tail quantiles, reflecting sensitivity to which model pairs occupy the low-margin regime.

\subsection{Same-Budget Detectability Across Protocols}
\label{app:robustness_same_dudget}
\begin{table}[h]
\centering
\caption{Same-budget detectability across protocols. Detectability is measured by
$z = |\hat p - 0.5| / \sqrt{\hat p(1-\hat p)/n}$ using subsamples of exactly $n$ judgments per pair.
Higher values indicate easier detection.}
\label{tab:same_budget_detectability}
\begin{tabular}{lcccc}
\toprule
Budget $n$ & Protocol & Count & $p_{10}(z)$ & Median$(z)$ \\
\midrule
50  & Arena    & 29{,}200 & 0.85 & 3.87 \\
50  & MT-Bench & 3{,}000  & 1.44 & 6.56 \\
100 & Arena    & 20{,}400 & 1.21 & 6.09 \\
100 & MT-Bench & 3{,}000  & 2.47 & 9.27 \\
\bottomrule
\end{tabular}
\end{table}

MT-Bench stochastically dominates Arena in detectability at matched budgets, consistent with reduced prompt-induced heterogeneity.

\subsection{Tie Frequency and Sensitivity to Tie Handling}
\label{app:ties}

Human preference data frequently includes ties, reflecting judgments where annotators perceive no clear winner. Our main analysis discards ties and conditions on decisive outcomes, corresponding to a \emph{best-case statistical} setting in which each retained judgment provides maximal directional information. Here we examine (i) where ties occur in practice and (ii) how alternative tie encodings affect estimated preference margins and detectability.

\textbf{Where ties occur.}
We analyze tie frequency among well-sampled Chatbot Arena model pairs (at least 200 decisive judgments per pair). The median tie rate across such pairs is $11.4\%$, with a 90th percentile of $22.9\%$. Ties are substantially more frequent in low-signal comparisons: among pairs with decisive-only margins $|\hat{\delta}| \le 0.1$, the median tie rate rises to $21.1\%$. Across well-sampled pairs, tie rate is strongly negatively correlated with estimated margin ($\mathrm{corr} = -0.575$), indicating that ties concentrate precisely in low-signal regimes.

\textbf{Sensitivity to tie encoding.}
We recompute preference margins under three reasonable tie encodings: (i) \emph{drop} ties (main analysis), (ii) \emph{half-vote}, treating ties as $0.5$ votes for each side, and (iii) a \emph{pessimistic} encoding that counts ties against the canonical first model. Including ties systematically reduces estimated margins at the lower tail. For example, the 10th-percentile margin decreases from $0.133$ under the drop-ties encoding to $0.105$ (half-vote) and $0.093$ (worst-case), increasing the implied number of judgments required for $90\%$ detection power from $186$ to $300$ and $385$, respectively.

Conditioning on the near-tie regime ($|\hat{\delta}| \le 0.1$), alternative encodings increase the number of affected model pairs but do not alter the qualitative conclusion: median margins remain small and implied evaluation budgets remain on the order of thousands of judgments across all encodings. Overall, alternative tie treatments do not rescue detectability; if anything, they further tighten feasibility constraints. Discarding ties therefore yields optimistic, best-case estimates of evaluability.

\subsection{Pick-a-Pic: Non-Chat Image Preference Evaluation}
\label{app:pickapic}

To assess whether low-signal regimes are specific to chat-based text generation, we analyze Pick-a-Pic, a large-scale human image preference dataset \citep{Kirstain2023PickAPic}. Pick-a-Pic consists of pairwise human judgments over images generated by different text-to-image models for the same prompt, with repeated annotations enabling estimation of preference margins at the prompt--model-pair level.

\textbf{Dataset.} We use the Pick-a-Pic dataset released by \citet{Kirstain2023PickAPic}, accessed via the Hugging Face mirror \citep{pickapic_hf}, which contains pairwise human image preferences without image payloads. We follow the same analysis pipeline as in the main text. For each prompt and unordered model pair, we aggregate decisive human judgments to estimate a Bernoulli preference parameter $\hat p$ and corresponding margin $|\hat p - 0.5|$. Judgments in which annotators reported no perceived difference between images are discarded, yielding a best-case setting consistent with our treatment of ties elsewhere in the paper.

\begin{figure}[ht]
    \centering
    \includegraphics[width=0.5\linewidth]{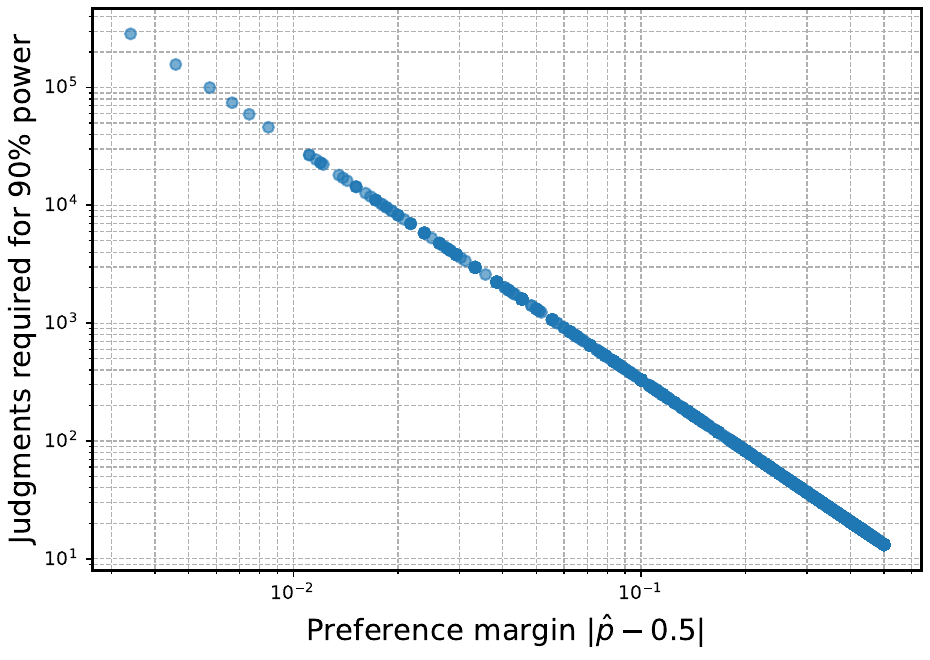}
    \caption{Pick-a-Pic image preference feasibility. Each point corresponds to a prompt--model-pair comparison, plotting the estimated preference margin $|\hat p - 0.5|$ against the implied number of judgments required for $90\%$ detection power.}
    \label{fig:pickapic}
\end{figure}

\textbf{Results.} Among the 5,016 prompt--model-pair comparisons with at least 10 decisive judgments, 710 comparisons (14.2\%) fall into the $|\hat{\delta}| \le 0.10$ regime. Within this near-tie subset, the median estimated margin is 0.045, implying a median evaluation budget of approximately $1.6 \times 10^3$ decisive judgments to achieve 90\% detection power (via Equation~\ref{eq:closed_form_simplified}). Figure~\ref{fig:pickapic} plots the estimated margins against implied budgets for all well-sampled comparisons, showing substantial heterogeneity in required evaluation budgets.

\textbf{Interpretation.} Although image preferences are often considered higher-signal than open-ended text evaluation, Pick-a-Pic exhibits the same qualitative feasibility constraints observed in chat-based benchmarks: a nontrivial fraction of human preference comparisons remain statistically expensive to resolve, even outside conversational or textual evaluation settings. The 14.2\% near-tie prevalence demonstrates that feasibility limits persist across modalities, confirming that these challenges are fundamental to human preference evaluation rather than artifacts of text-based assessment.

\subsection{BigCodeArena: Code Generation with Execution Feedback}
\label{app:bigcodearena}

To assess whether low-signal regimes are specific to subjective text and image 
evaluation, we analyze BigCodeArena \citep{zhuo2025bigcodearena}, a code 
generation evaluation platform with execution feedback. BigCodeArena extends 
Chatbot Arena by enabling on-the-fly code compilation, execution, and interactive 
UI testing. Human evaluators judge pairwise model comparisons based on both 
source code and execution results, which may include interactive applications, 
visualizations, or textual output.

\textbf{Dataset.} We use the publicly released preference dataset containing 
4731 multi-turn conversations with human votes across 10 LLMs spanning 10 
programming languages and 8 execution environments (React, PyGame, Mermaid, etc.). 
Following the filtering criteria in \citet{zhuo2025bigcodearena}, we restrict to 
conversations with at least 2 user-model exchanges and verified code execution. 
For each model pair, we aggregate decisive votes ("Model A Better" or "Model B 
Better," discarding "Tie" and "Both Bad") to estimate the preference margin 
$\hat{\delta} = |\hat{p} - 0.5|$.

\begin{figure}[ht]
\centering
\includegraphics[width=0.5\textwidth]{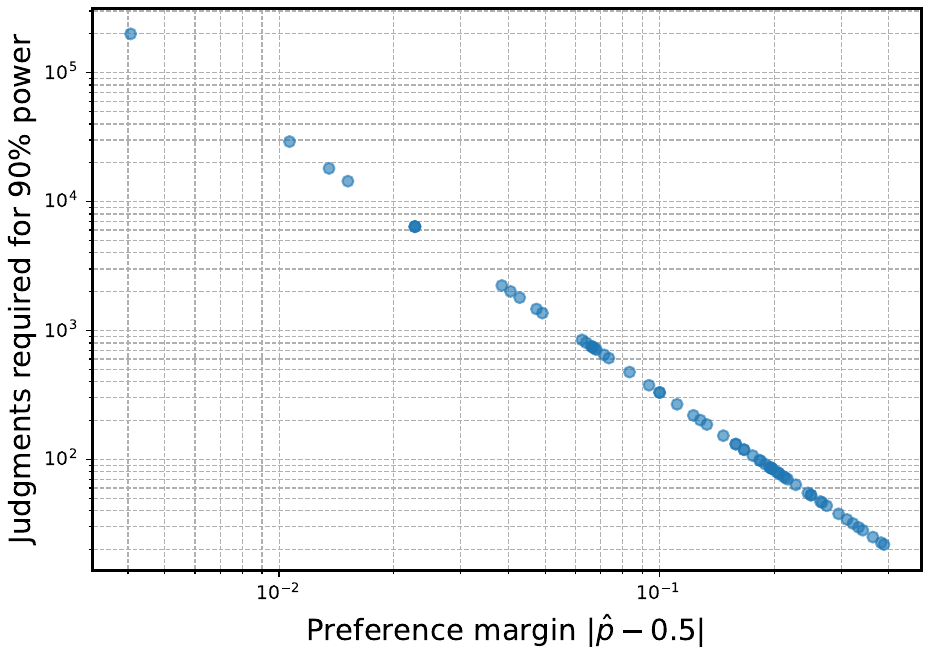}
\caption{BigCodeArena code generation feasibility. Each point corresponds to a 
model pair, plotting the estimated preference margin $|\hat{p} - 0.5|$ against 
the implied number of judgments required for 90\% detection power. Despite 
execution feedback enabling objective correctness verification, 41\% of pairs 
fall into the $|\delta| \leq 0.10$ regime, requiring more than 300 judgments 
for 90\% power at $\alpha = 0.05$.}
\label{fig:bigcodearena}
\end{figure}

\textbf{Results.} Among the 66 model pairs with at least 10 decisive judgments,  27 pairs (41\%) fall into the $|\hat{\delta}| \leq 0.10$ regime. Within this near-tie subset, the median estimated margin is 0.049, implying a median evaluation budget of approximately 1,364 decisive judgments to achieve 90\% detection power (via Equation~\ref{eq:closed_form_simplified}). Figure~\ref{fig:bigcodearena} plots the estimated margins against implied budgets for all well-sampled pairs, showing substantial heterogeneity in required evaluation budgets.

\textbf{Interpretation.} The 41\% near-tie prevalence in BigCodeArena substantially exceeds the 17.6\% observed in Chatbot Arena (Table~\ref{tab:near-tie}), despite the availability of execution feedback. This suggests that code evaluation remains intrinsically difficult even when partial ground truth is available. While execution enables objective verification of functional correctness, human preferences in code generation depend on multiple dimensions—including code readability, maintainability, efficiency, and UI/UX quality—that are not fully captured by execution alone. These findings validate our theoretical framework in a setting where objective and subjective evaluation signals coexist, and underscore that feasibility limits persist across diverse evaluation modalities.

\vspace{2pt}


\end{document}